\documentclass[11pt]{article}
    \usepackage{amsmath,amssymb,amsfonts,amsthm,epsfig}
    \usepackage[usenames,dvipsnames]{xcolor}

    \usepackage{bm,xspace}
    
    \usepackage[authoryear]{natbib}

    \usepackage{tikz}
    
    \usepackage{cancel}
    \usepackage{fullpage}
    \usepackage{liyang}
    \usepackage{framed}
    \usepackage{verbatim}
    \usepackage{enumitem}
    \usepackage{array}
    \usepackage{multirow}
    \usepackage{afterpage}
    \usepackage{mathrsfs}  
    \usepackage{cleveref}
    
    \usepackage[normalem]{ulem}

    \usepackage{caption} 

    \usepackage{todonotes}
    
    \makeatletter
    \newtheorem*{rep@theorem}{\rep@title}
    \newcommand{\newreptheorem}[2]{
    \newenvironment{rep#1}[1]{
     \def\rep@title{#2 \ref{##1}}
     \begin{rep@theorem}\itshape}
     {\end{rep@theorem}}}
    \makeatother
    \theoremstyle{plain}

    
    
    \newcommand{\ignore}[1]{}

    \def\colorful{1}

    \ifnum\colorful=1

    \fi
    \ifnum\colorful=0

    \fi

    \usepackage{boxedminipage}

    \newreptheorem{theorem}{Theorem}
    \newtheorem*{theorem*}{Theorem}
    \newreptheorem{lemma}{Lemma}
    \newreptheorem{proposition}{Proposition}
    \newtheorem*{noclaim*}{Claim}

    \usepackage{dsfont}

    \renewcommand{\R}{\mathds{R}} 
    \newcommand{\cD}{{\cal D}}
    \newcommand{\cF}{{\cal F}}
    \newcommand{\Flin}{\cF_{\mathrm{lin}}}
    
    \newcommand{\ocD}{\overline{{\cal D}}}
    \newcommand{\ty}{\tilde{y}}

\begin{document}

    \title{\bf The perils of being unhinged: \\
           On the accuracy of classifiers minimizing \\
           a noise-robust convex loss}
    
    \author{Philip M. Long \\
            Google \\
            plong@google.com \\
            \and
            Rocco A. Servedio \\
            Columbia University \\
            rocco@cs.columbia.edu \\
            }
    
    \date{}
        \maketitle
    
    



%

    \begin{abstract}
        \citet{VRMW15}~introduced a notion of  convex loss functions being robust to random classification noise, and established that the ``unhinged'' loss function is robust in this sense. In this note we study the accuracy of binary classifiers obtained by minimizing the unhinged loss, and observe that even for simple linearly separable data distributions, minimizing the unhinged loss may only yield a binary classifier with 
        accuracy no better than random guessing.
    \end{abstract}
    
    \sloppy
    \section{Introduction} \label{sec:intro}
    
    As van Rooyen et al.\ 
    noted
    in the first sentence of the abstract of \citet{VRMW15}, 
    ``Convex potential minimisation is the \emph{de facto} approach to binary classification.'' Given the ubiquity of this approach, it is natural to study its abilities and limitations in the presence of noise, and indeed this is the subject of many works 
    \citep[see][]{zhang2004statistical,bartlett2006convexity,LS10,ManwaniSastry13,NDRT13,VRMW15,ghosh2017robust}).
    
    The aim of this note is to clarify the connection between minimizing a convex potential function which is ``robust to classification noise'' in the sense of \citet{VRMW15}, and learning (i.e.~performing accurate classification).
    
    \medskip
    
    \noindent {\bf Background.}
    Motivated by the observation that the popular AdaBoost algorithm (which works by minimizing the (convex) exponential potential function) can have empirically poor classification accuracy when run on noisy data \citep{Dietterich:00,FreundSchapire:96b,MaclinOptiz:97},  \citet{LS10} studied the performance of classification algorithms which work by minimizing a convex potential function in settings where linearly separable data is contaminated with random classification noise (RCN). The main result of \citet{LS10} is a proof that for a certain simple learning problem corresponding to a ``clean'' data distribution ${\cal D}_1$ that is linearly separable with a margin, for \emph{any} ``convex potential function'' $\phi$, minimizing $\phi$ over all linear combinations of base features in the presence of random classification noise only yields a binary classifier with an error rate of 1/2 under the clean distribution ${\cal D}_1$.  (Here a ``convex potential function'' is a convex function $\phi: \R \to \R$ satisfying certain mild conditions which we detail in \Cref{def:cpf}.) This is in sharp contrast with the fact that, 
    in the noise-free setting of a 
    data distribution that is linearly separable with a margin,
    driving the potential to zero leads to a 
    perfectly accurate binary classifier.
    
    In an effort to address the discouraging negative result of \citet{LS10}, \citet{VRMW15} considered a weakening of the \citep{LS10} conditions for a convex potential function. In particular, they allow such functions $\phi$ to take 
    negative values 
    (which is disallowed by the definition of Long and Servedio). We refer to a function satisfying the condition of \citet{VRMW15} as a ``relaxed convex potential function.''  
    
 The main result of \citet{VRMW15} is that they
    propose a certain relaxed convex potential function, which we denote $\phi^\star$, and prove that it is ``RCN-robust''.\footnote{\citet{VRMW15} uses the term ``SLN-robust'', where the acronym stands for Symmetric Label Noise.}  We give a formal definition of RCN-robustness in \Cref{sec:preliminaries}, but intuitively it means that a minimizer of this potential function (minimizing over all linear combinations of base features) under random classification noise performs no worse than a minimizer obtained with no random classification noise.  \citet{VRMW15} also define a notion of ``strong RCN-robustness'' and show that their $\phi^\star$ is the unique relaxed convex potential function which satisfies strong RCN-robustness.
    
    \medskip
    \noindent {\bf This note.} The purpose of the present note is to discuss the \emph{accuracy} of the classifier obtained by minimizing the relaxed convex potential function $\phi^\star$ of \citet{VRMW15}.   
    Our main observation is that, for a simple learning problem corresponding to a certain ``clean'' data distribution ${\cal D}_2$ that is linearly separable with a margin, minimizing $\phi^\star$ over all bounded-norm linear combinations of base features \emph{even when there is no random classification noise} only yields a binary classifier with an error rate of 1/2.  Since, as shown by \citet{VRMW15}, $\phi^\star$ is the unique strong RCN-robust relaxed convex potential function, this means that minimizing any strong RCN-robust relaxed convex potential function in this noise-free scenario may only yield a binary classifier with an error rate of 1/2,
    which can be obtained through random guessing.
    
    Our observation is consistent with the result of \citet{VRMW15}  that $\phi^\star$ is RCN-robust, since, 
    informally,
    that condition only states that ``you don't do any worse when there is RCN than when there is no RCN.''  Our example demonstrates even when there is no 
    noise,
    the accuracy of the binary classifier obtained by minimizing $\phi^\star$ may be only 1/2, and of course the accuracy is no worse than this when there actually is noise.
    
    \section{Preliminaries} \label{sec:preliminaries}
    
    \subsection{Background: the negative result of \citet{LS10}  for convex potential functions}
    
    \noindent {\bf Convex potential functions.}
    We recall the following definition which is central to the work of \citet{LS10}:
    
    \begin{definition} [\citep{LS10}, Definition~1] \label{def:cpf}
    A function $\phi: \R \to \R$ is a \emph{convex potential function} if it satisfies the following:
    \begin{enumerate}
    \item $\phi \in C^1$ (i.e.~$\phi$ is differentiable and $\phi'$ is continuous) and $\phi$ is convex and nonincreasing; and
    \item $\phi'(0)<0$ and $\lim_{x \to \infty}\phi(x)=0$ (hence $\phi$ is everywhere non-negative).
    \end{enumerate}
    \end{definition}
    

  A number of potential functions used in the literature fit this definition, including the exponential potential function used by AdaBoost \citep{FreundSchapire:97}, the mixed linear/exponential potential function used by MadaBoost \citep{DomingoWatanabe:00}, and the logistic function  used by LogitBoost \citep{FHT:98}; see Table~1.
  
  \begin{table} \label{table:potentials}
  \begin{tabular}{ |p{4.7cm}||p{4.9cm}|p{3.9cm}|  }
 \hline
Potential function & Reference &Satisfies \Cref{def:cpf}?\\
 \hline
 Exponential: & & \\
 $\phi(z) = e^{-z}$   & \citep{FreundSchapire:97}    &Yes\\
 \hline
 Mixed linear/exponential: & & \\
 $\phi(z) = 
 \begin{cases}
1-z & \text{~if~}z \leq 0\\
 e^{-z} & \text{~if~}z>0\\
 \end{cases}$&   \citep{DomingoWatanabe:00}  & Yes \\
 \hline
 Logistic: & & \\
 $\phi(z) = \ln(1+e^{-2z})$ &\citep{FHT:98} & Yes\\
 \hline
 Hinge: & &  \\
 $\phi(z)= \max\{0,1-z\}$    &
 \citep{gentile1998linear}
 & No\\
 \hline
 Unhinged: & & \\
 $\phi(z)=1-z$ & \citep{VRMW15} & No\\
 \hline
\end{tabular}
\caption{Some commonly used potential functions.}
\label{t:loss}
\end{table}

    \medskip
    \noindent {\bf Linearly separable learning problems.}
    One of the simplest models for binary-labeled data over  $\R^d$ is that of data which is \emph{linearly separable with a margin}. 
    A ``clean'' probability distribution ${\cal D}$ over $\R^d \times \{-1,1\}$ is linearly separable with margin $\gamma>0$ if  there is a target weight vector $w=(w_1,\dots,w_d) \in \R^d$ such that 
    \[
    \Prx_{(\bx,\by) \sim {\cal D}}\left[{\frac {\by (w \cdot \bx)}{|w_1| + \cdots + |w_d|}} < \gamma \right]=0.
    \]
    A very standard learning approach for such a setting is to choose a hypothesis weight vector $v=(v_1,\dots,v_d) \in \R^d$ with the aim of minimizing the ``global'' potential function 
    \begin{equation} \label{eq:gpf}
    P_{\phi, \cD}(v) := \Ex_{(\bx,\by) \sim {\cal D}}\left[\phi(\by(v\cdot \bx))\right].
    \end{equation}
    (Of course, given a finite sample of draws from ${\cal D}$, this is typically done by minimizing the corresponding expectation over the sample.)  It is well known that if ${\cal D}$ is linearly separable with margin $\gamma>0$, then for 
    a range of different choices of the convex potential function $\phi$ (including the AdaBoost, MadaBoost and LogitBoost potential functions described above),
    greedy iterative algorithms that perform coordinatewise gradient descent to minimize $P_{\phi,\cD}$ 
    will drive the misclassification error 
    $\Pr_{(\bx,\by) \sim {\cal D}}[\by \neq \sign(v \cdot \bx)]$ 
    to zero.
    Indeed, the AdaBoost \citep{FreundSchapire:97}, MadaBoost \citep{DomingoWatanabe:00} and LogitBoost \citep{FHT:98}
    boosting algorithms correspond precisely to greedy coordinatewise gradient descent procedures of this sort; see the work of \citet{MBB+:99} for details.
    
    \medskip
    \noindent {\bf Learning problems with random classification noise.}
    Let ${\cal D}$ be a\ignore{linearly separable} data distribution over $\R^d \times \bits$ as described above.  The \emph{$\eta$-RCN corrupted} version of ${\cal D}$ is the following distribution $\overline{{\cal D}}_\eta$ over $\R^d \times \bits$: a draw from $\overline{{\cal D}}_\eta$ is obtained by drawing $(\bx,\by) \sim {\cal D}$ and flipping the label $\by$ with probability $\eta$.
    
    \medskip
    \noindent {\bf The negative result of \citet{LS10}.}
    The main result of \citet{LS10} is that there is \emph{no} convex potential function such that minimizing $\phi$ on the $\eta$-RCN corrupted distribution $\overline{\cal D}_\eta$ will succeed in achieving nontrivial classification accuracy:
    
    \begin{theorem} \label{thm:LS08}
    Fix any noise rate $0 < \eta < 1/2$ and any convex potential function $\phi$. There is a distribution ${\cal D}$ over $\R^2 \times \{-1,1\}$ (in fact the distribution ${\cal D}$ is supported on three points in the unit disc) and a margin parameter $\gamma > 0$ such that (a) ${\cal D}$ is linearly separable with margin $\gamma$,
    but (b) any weight vector 
    $v$ 
    which minimizes $\Ex_{(\bx,\by) \sim \overline{{\cal D}}_\eta}\left[\phi(\by(v\cdot \bx))\right]$ has $\Pr_{(\bx,\by) \sim {\cal D}}[\by \neq \sign(v \cdot \bx)]=1/2.$
    \end{theorem}
    (See Appendix~\ref{a:minimum} for a proof that for any convex potential function $\phi$, the minimizer analyzed in
    Theorem~\ref{thm:LS08} exists.)
    
    \subsection{Relaxed convex potential functions: a new hope?}
    
    Motivated by the goal of circumventing the negative result of \Cref{thm:LS08}, \citet{VRMW15} consider a relaxed form of \Cref{def:cpf}:
    
    \begin{definition} \label{def:rcpf}
    A function $\phi: \R \to \R$ is a \emph{relaxed} convex potential function if it satisfies the following:
    \begin{enumerate}
    \item $\phi \in C^1$ (i.e.~$\phi$ is differentiable and $\phi'$ is continuous) and $\phi$ is convex and nonincreasing; and
    \item $\phi'(0)<0$.
    \end{enumerate}
    \end{definition}
    The only difference between \Cref{def:cpf} and \Cref{def:rcpf} is that the latter does not require $\lim_{x \to \infty} \phi(x) = 0$; a relaxed convex loss function may take (arbitrarily large magnitude) negative values.  \citet{VRMW15} exploit this flexibility by proposing the following simple potential function, which they call the ``unhinged loss'':
    \[
    \phi^\ast(z) = 1 - z.
    \]
    It is trivial to verify that $\phi^\star$ satisfies \Cref{def:rcpf} and hence is a valid relaxed convex potential function. 
    (Note, also, that if the simpler
    $\phi(z) = -z$ is used instead, all gradients and minima are 
    unaffected.)
    We note that the unhinged loss is a member of the class of \emph{symmetric} loss functions, which satisfy $\phi(z) + \phi(-z) = $constant; such loss functions have been studied by a number of authors, see e.g. \citet{CLS19,GMS15}.

    \medskip
    \noindent {\bf RCN-robustness.}
    \citet{VRMW15} analyze the relaxed convex potential function $\phi^\star$ through the lens of a new notion which 
    we will
    call \emph{RCN-robustness}.
    Their definition (Definition~1 of \citet{VRMW15}) applies to a general pair $(\ell,{\cal F})$ where $\ell$ is a  loss function and ${\cal F}$ is a class which may consist of any collection of functions mapping a domain $X$ to $\R$. 
    
    Informally, a pair $(\phi,{\cal F})$ is RCN-robust if minimizing $\phi$ over ${\cal F}$  on noise-free data gives the same binary classification performance as minimizing $\phi$ over ${\cal F}$ on RCN-contaminated data.  More precisely, we have the following:
    
    \begin{definition} \label{def:rcn-robust}
    Let $\cF$ be a set of real-valued functions over $\R^d$ and let $\phi$ be a potential function. The pair $(\phi,\cF)$ is  
    is said to be \emph{RCN-robust} if the following holds: Let ${\cal D}$ be any distribution over $\R^d \times \bits$ and let $0<\eta<1/2$ be any noise rate.
    If $f$ is a minimizer of 
    $ \Ex_{(\bx,\by) \sim {\cal D}}\left[\phi(\by(f(\bx)))\right]$ over $f \in \cF$
    and  $g$ is the minimizer of $ \Ex_{(\bx,\by) \sim \overline{\cal D}_\eta}\left[\phi(\by(g( \bx)))\right]$ over $g \in \cF$, then 
    \begin{equation} \label{eq:robust}
    \Prx_{(\bx,\by) \sim {\cal D}}[\by \neq \sign(f(\bx))] = 
    \Prx_{(\bx,\by) \sim {\cal D}}[\by \neq \sign(g(\bx))].
    \end{equation}
    \end{definition}
    
    \citet{VRMW15} specialize \Cref{def:rcn-robust} to the function class $\Flin$ of all linear functions 
    $x \mapsto v \cdot x$ 
    from $\R^d \to \R$ (see Section~3.2 of their paper). However, a problem with \Cref{def:rcn-robust} for this function class is that $ \Ex_{(\bx,\by) \sim {\cal D}}\left[\phi(\by(v\cdot \bx))\right]$ may not
    have a minimum.  This is not merely a technicality.
    In fact, for standard loss functions
    such as the logistic loss or the exponential loss,
    for any linearly separable distribution $\cal D$,
    $ \Ex_{(\bx,\by) \sim {\cal D}}\left[\phi(\by(v\cdot \bx))\right]$
    does not have a minimum, informally, because scaling up
    $v$ increases all of the margins, which decreases all of the
    losses.\footnote{Implicit bias research analyzes the effect
    of the algorithm that drives $ \Ex_{(\bx,\by) \sim {\cal D}}\left[\phi(\by(v\cdot \bx))\right]$ to zero on the classification
    behavior of the limiting classifier. 
  Different algorithms lead
    to markedly different limiting classifiers \citep{DBLP:conf/icml/Telgarsky13,soudry2018implicit,ji2019implicit} .
    }  \citet{VRMW15} interpret \Cref{thm:LS08} as saying that for $d \geq 2$, the pair  $(\phi,\Flin)$ cannot be RCN-robust for any convex potential function $\phi$ (see Proposition~1 of Section~3.2 of their paper), but the fact that the minimizer typically
    doesn't exist in the absence of noise interferes with this interpretation.
    The unhinged loss also 
    cannot be minimized over $\Flin$, since
    by scaling up the weight vector of any linear separator, the unhinged loss can
    achieve an arbitrarily large negative value.  
    
    \citet{VRMW15} also consider the class $\cF_{\mathrm{lin},r}$
    of all linear functions whose weight vector has length at most $r$.
    They prove the following:
    
     
    \begin{theorem} [\citep{VRMW15}, Section~5.1] \label{thm:rcn-robust}
    For all $d$ and all $r > 0$,
    $(\phi^\star, \cF_{\mathrm{lin},r})$ is RCN-robust.
    \end{theorem}
    
    \citet{VRMW15} further establish a number of additional properties about the unhinged loss $\phi^\star$; most of these will not concern us, but one simple property, which we now explain, is relevant to our discussion in \Cref{sec:example}.
    As above let $v \in \R^d$ be the minimizer of  $\Ex_{(\bx,\by) \sim {\cal D}}\left[\phi(\by(v\cdot \bx))\right]$ subject to $|| v || \leq r$ and let
    $v' \in \R^d$ be the minimizer of  $\Ex_{(\bx,\by) \sim \overline{\cal D}_{\eta}}\left[\phi(\by(v'\cdot \bx))\right]$ subject to $|| v' || \leq r$.
    \citet{VRMW15} make the straightforward but useful observation that $v$ is the vector corresponding to  a ``nearest centroid classifier'' (see
    \citet{DBLP:journals/siamcomp/Servedio02}, \citep{Tibshirani02}, p. 181 of \citep{Manning08}, and Section~5.1 of \citet{STC04}), i.e.~we have
    \begin{equation} \label{eq:centroid}
    v =  \alpha \Ex_{(\bx,\by) \sim {\cal D}}[\by \bx]
    \end{equation}
    for a suitable rescaling factor $\alpha$,
    and furthermore that $v = v'$ (this holds since the values of 
    $\Ex_{(\bx,\by) \sim {\cal D}}\left[\phi^\star(\by(w\cdot \bx))\right]$
    and
    $\Ex_{(\bx,\by) \sim \overline{{\cal D}}_\eta}\left[\phi^\star(\by(w\cdot \bx))\right]$
    are linearly related with a slope of $1-2\eta>0$).

    
    

    \section{A separable learning problem where minimizing the unhinged loss on clean data yields a poor classifier}
    \label{sec:example}
    
    \begin{figure}
    \hskip 2in
    \begin{tikzpicture}
    \draw[dotted] (0,0)--(6,0);
    \draw[dotted] (1,-1)--(1,5);
    
    \draw[->] (1,0)--(1.31,0.1691);
    \node[] at (1.4,0.4091) {$v$};

    \draw (0.1545,1.3)--(1.8455,-1.3);
    
    \draw[fill] (5,0) circle [radius =0.04cm];
    \node[] at (5.0,0.3) {(1,0)};
    
    \draw[fill] (1.32,3.98) circle [radius =0.04cm];
    \node[] at (1.32,4.28) {$(\gamma,\sqrt{1-\gamma^2})$};
    
    \draw (1.32,-0.64) circle [radius =0.07cm];
    \node[] at (0.62,-1) {$(\gamma,-2\gamma)$};

    \end{tikzpicture}
    \caption{The distribution ${\cal D}$ over $\R^2 \times \bits$. All examples have label $+1$; the two examples with weight $1/4$ are depicted with small filled circles and the example with weight $1/2$ is depicted with a larger unfilled circle.}
    \end{figure}
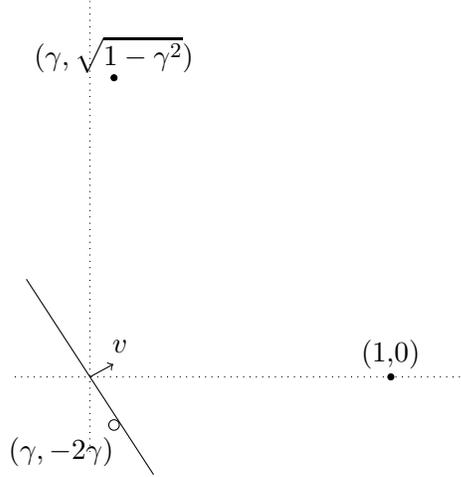

    In this section we observe that while the unhinged loss $\phi^\star$ is strongly robust, there are simple linearly separable data distributions for which minimizing $\phi^\star$ over all functions in $\cF_{\mathrm{lin},r}$ \emph{even in the absence of random classification noise} only yields a binary classifier with an error rate of 1/2. So while (\ref{eq:robust}) is satisfied, it holds because both error rates are equal to 1/2.
    
    We illustrate this with the linearly separable learning scenario which is depicted in Figure~1. The distribution ${\cal D}$ over $\R^2 \times \bits$ is as follows: given a parameter $0 < \gamma < 0.0901$,
    
    \begin{itemize}
    
    \item ${\cal D}$ puts weight $1/4$ on the labeled example $x^{(1)}:=(1,0),y^{(1)} :=1$;
    \item ${\cal D}$ puts weight $1/4$ on the labeled example $x^{(2)}:=(\gamma,\sqrt{1-\gamma^2}),y^{(2)} :=1$;
    \item ${\cal D}$ puts weight $1/2$ on the labeled example $x^{(3)}:=(\gamma,-2\gamma),y^{(3)} :=1$.
    
    \end{itemize}
    It is clear that ${\cal D}$ is linearly separable with margin $\gamma$.
    By (\ref{eq:centroid}), the vector in $\R^2$ which minimizes $\Ex_{(\bx,\by) \sim {\cal D}}\left[\phi(\by(v\cdot \bx))\right]$
    points in the direction of 
    \[
    v=(v_1,v_2) = \left({\frac 1 4} + {\frac {3 \gamma} 4}, {\frac{\sqrt{1-\gamma^2}} 4} - \gamma \right).
    \]
    For $0 < \gamma < 0.0901$ we have $v \cdot x^{(3)} < 0$ and hence $\sign(v \cdot x^{(3)}) \neq y^{(3)}$, so the LHS of (\ref{eq:robust}) is $1/2$. Since $v'=v$ the RHS of (\ref{eq:robust}) is also $1/2.$
    
    \ignore{
    
    
    }
    
    \section{Implicit bias}
    
    This section includes a couple of observations
    about the implicit bias of algorithms that
    iteratively reduce the unhinged loss.
    Analogous results have been
    obtained for other loss functions
    \citep{DBLP:conf/icml/Telgarsky13,soudry2018implicit,ji2019implicit}.
    
    \subsection{Gradient descent}
    
    Recall that the unhinged loss function is defined to be $\phi^\ast(z) = 1 - z.$
    If $\cD$ is uniform over
    $(x_1,y_1),...,(x_n,y_n) \in \R^d \times \{ -1, 1 \}$, then
    for any $v \in \R^d$ the gradient of the unhinged loss at
    $v$ is $-\sum_{i=1}^n y_i x_i$ (note that this does not depend on $v$). Thus, if the unhinged loss is minimized
    by gradient descent starting with an initial solution of
    $0$, all iterates are multiples of 
    $\sum_{i=1}^n y_i x_i$.  If the initial solution is
    $v_0$, then, after $T$ updates with step size
    $\eta$, the weight vector is $v_0 + \eta T \sum_{i=1}^n y_i x_i$.
    As $T$ goes to infinity, the angle between this weight
    vector and $\sum_{i=1}^n y_i x_i$ goes to zero.
    
    \subsection{Coordinate descent}
    
    As mentioned earlier, popular boosting
    algorithms can be viewed as coordinate descent
    on a convex potential function, which works by repeatedly
    finding the coordinate axis with
    the steepest descent direction  and
    making an update in that direction.
    Informally, the
    unhinged loss rewards increasing the margin
    $y v \cdot x$ of a correctly classified example
    $(x,y)$
    as much as increasing the negative margin of an incorrectly 
    classified example, but increasing the margin of a correctly classified example does not make progress towards overall classification accuracy. (In contrast, the tendency of the exponential loss
    to place more importance on misclassified examples is key
    to AdaBoost's ability to boost.)  
    If we denote the components of $x_i$ by $x_{i,1},...,x_{i,d}$,
    since the  gradient of the unhinged loss is the
    same for all $v$, when it is minimized
    by coordinate descent starting from the zero weight
    vector all of its iterates will only have nonzero components on
    members of
    $\mathrm{argmax}_j \sum_i y_i x_{ij}$.  
    (If there is not a tie for the best weak learner, this will be
    a single component.)  
    From a boosting point of view,
    a boosting algorithm based on the unhinged loss allows
    a weak learner to keep returning the same (weak)
    hypothesis.
    
    \section{Discussion}
    
    While the unhinged loss is noise-tolerant in a sense, minimizing it can fail to find an accurate classifier on data that is linearly separable with a large margin.  On the other hand, minimizing the unhinged loss has been found to yield reasonable accuracy on natural data \citep[see][]{patrini2017making,charoenphakdee2019symmetric}.  This is not entirely
    unexpected, since, when learning linear models, minimizing the unhinged loss is closely related to performing Naive Bayes classification \citep{domingos1997optimality,NJ01}, using a spherical Gaussian to model the class-conditional distributions.  
    
    Given our results, one natural goal for future work is to study whether there are conditions on potential functions which achieve an attractive tradeoff between noise-robustness and usefulness for learning (in the sense that minimizing the potential function yields an accurate classifier).  
    Tools developed for studying
    Fisher consistency \citep{fisher1922mathematical},
    consistent loss functions \citep{zhang2004statistical},
    classification calibration \citep{bartlett2006convexity}
    and $H$-consistency \citep{long2013consistency} may be
    useful for this.  In particular, it would be interesting to investigate symmetric potential functions (see e.g. \citet{CLS19,GMS15}) and the multiclass setting (see e.g. \citep{ghosh2017robust,ZM18}) in light of this question.

    \appendix
    
    \section{A minimizer exists for noisy data}
    \label{a:minimum}
    
    In this appendix we show that 
    for all convex potential functions $\phi$, all finite-covariance distributions $\cD$ over $\R^d$,
    and all $\eta \in (0, 1/2)$, the function $P_{\phi,\ocD_{\eta}}$
    has a minimum.  
    (Recall from (\ref{eq:gpf}) that
     $P_{\phi, \ocD_{\eta}}(v) := \Ex_{(\bx,\by) \sim \ocD_{\eta}}\left[\phi(\by(v\cdot \bx))\right].$)

    We recall some useful background.
    
    \begin{definition}
    For any $a$, the set $\{ x : f(x) \leq a \}$
    is a {\em level set} for $f: \R^d  \to \R$.
    \end{definition}
    
    \begin{definition}
    A nonzero vector $u \in \R^d$ is a {\em direction of
    recession} for a function $f$
    if, for all
    nonempty level sets $L$ of $f$, there exists some
    $x_0$ such that $x_0 + \lambda u \in L$ for all
    $\lambda \geq 0$.
    (Informally, all non-empty level sets of $f$
    extend infinitely in the $u$ direction.)
    \end{definition}
    
    \begin{lemma}[\citep{rockafellar2015convex}, Theorem 27.1]
    \label{l:recession}
    The set of minima
    of a 
    continuous convex function $f$ is nonempty and bounded
    iff $f$ does not
    have any direction of recession.
    \end{lemma}
    
    \begin{definition}
    Say that a probability distribution
    $\cD$ over $\R^d \times \{-1, 1\}$ has {\em finite covariance} if,
    for all unit length $u \in \R^d$,
    $\E_{(\bx,\by)\sim \cD}[( u \cdot \bx)^2]$ exists.
    \end{definition}
    
    Now we are ready to analyze $P_{\phi,\ocD_{\eta}}$.
    \begin{proposition}
    \label{p:minimum}
    For any convex potential function $\phi$,
    for any $\eta \in (0,1/2)$, for any finite-covariance distribution
    $\cD$ over $\R^d$, $P_{\phi,\ocD_{\eta}}$ has a minimum.
    \end{proposition}
    \begin{proof}
    First, we may assume without loss of generality
    that, for all unit length $u \in \R^d$,
    \begin{equation}
    \label{e:width}
    \Ex_{(\bx,\by)\sim \cD}[| u \cdot \bx|] > 0
    \end{equation}
    since otherwise $u \cdot \bx=0$ almost surely, and $P_{\phi,\ocD_{\eta}}(v)$
    is unaffected by projecting $v$ onto
    the subspace of $\R^d$ orthogonal to $u$.
    
    Assume for contradiction that some $u$ is a direction
    of recession for $P_{\phi,\ocD_{\eta}}$.  
    For any $x_0 \in \R^d$ and $\lambda \geq 0$, we have
    \begin{align*}
        &P_{\phi,\ocD_{\eta}}(x_0 + \lambda u) \\
        &= \Ex_{(\bx,\by) \sim \ocD_{\eta}} [\phi(\by ((x_0 + \lambda u) \cdot \bx))] \tag{def.~of $P_{\phi,\ocD_{\eta}}$}\\
    & = \Ex_{(\bx,\by) \sim \cD} [ \eta \phi(-\by ((x_0 + \lambda u) \cdot \bx))
                  + (1 - \eta) \phi(\by ((x_0 + \lambda u) \cdot \bx))
                  ] \tag{def.~of $\overline{{\cal D}}_\eta$}\\
    & \geq \Ex_{(\bx,\by) \sim \cD} [ \eta \max_{\ty \in \{-1, 1\}} \phi(\ty ((x_0 + \lambda u) \cdot \bx))
                  ] \tag{since $\eta<1/2$ and $\phi \geq 0$} \\
    & = \Ex_{(\bx,\by) \sim \cD} [ \eta \phi(-| (x_0 + \lambda u) \cdot \bx|)
                  ] \tag{since $\phi$ is nonincreasing} \\
    & \geq \eta \Ex_{(\bx,\by) \sim \cD}  [  \phi(0) - \phi'(0) \cdot| (x_0 + \lambda u) \cdot \bx |
                  ] \tag{since $\phi$ is convex and $\phi'(0) <0$} \\
    & \geq 
     \eta \left( \phi(0) - \phi'(0) \lambda 
       \Ex_{(\bx,\by) \sim \cD}  [ |  u \cdot \bx | ]
         +  \phi'(0)
          \Ex_{(\bx,\by)}[| x_0 \cdot \bx |
                  ] \right) \tag{triangle inequality, $\phi'(0) < 0$}.
    \end{align*}
    Thus $\lim_{\lambda \rightarrow \infty} P_{\phi,\ocD_{\eta}}(x_0 + \lambda u) = \infty$, which contradicts the assumption that $u$ is a direction of
    recession for $P_{\phi,\ocD_{\eta}}$.
    \end{proof}

  \bibliographystyle{plainnat}

    \bibliography{allrefs}

    \end{document}